\documentclass[twoside,leqno,onecolumn]{article}

\usepackage[letterpaper]{geometry}

\usepackage{siamproceedings}

\usepackage[table]{xcolor}
\usepackage[T1]{fontenc}
\usepackage{amsfonts}
\usepackage{graphicx}
\usepackage{epstopdf}
\usepackage{enumitem}
\usepackage{algorithmic}
\ifpdf

  \DeclareGraphicsExtensions{.eps,.pdf,.png,.jpg}
\else
  \DeclareGraphicsExtensions{.eps}
\fi

\newsiamremark{remark}{Remark}
\newsiamremark{hypothesis}{Hypothesis}
\crefname{hypothesis}{Hypothesis}{Hypotheses}
\newsiamthm{claim}{Claim}

\usepackage{amsopn}

\DeclareMathOperator*{\argmin}{arg\,min}
\usepackage{booktabs}

\usepackage{multirow}

\begin{document}

\title{\LARGE Low-Rank Compression of Pretrained Models via Randomized Subspace Iteration}
\author{Farhad Pourkamali-Anaraki\\Department of Mathematical and Statistical Sciences, University of Colorado Denver\\
Email: \email{farhad.pourkamali@ucdenver.edu}}


\date{}

\maketitle

\fancyfoot[C]{\thepage}



\begin{abstract} The massive scale of pretrained models has made efficient compression essential for practical deployment. Low-rank decomposition based on the singular value decomposition (SVD) provides a principled approach for model reduction, but its exact computation is expensive for large weight matrices. Randomized alternatives such as randomized SVD (RSVD) improve efficiency, yet they can suffer from poor approximation quality when the singular value spectrum decays slowly, a regime commonly observed in modern pretrained models.
In this work, we address this limitation from both theoretical and empirical perspectives. First, we establish a connection between low-rank approximation error and predictive performance by analyzing softmax perturbations, showing that deviations in class probabilities are controlled by the spectral error of the compressed weights. Second, we demonstrate that RSVD is inadequate, and we propose randomized subspace iteration (RSI) as a more effective alternative. By incorporating multiple power iterations, RSI improves spectral separation and provides a controllable mechanism for enhancing approximation quality.
We evaluate our approach on both convolutional networks and transformer-based architectures. Our results show that RSI achieves near-optimal approximation quality while outperforming RSVD in predictive accuracy under aggressive compression, enabling efficient model compression.
\end{abstract}

\section{Introduction.} 
Over the past decade, the field of deep learning has witnessed a rapid increase in the size of neural networks driven by the availability of large-scale datasets. This trend has led to the emergence of pretrained models, where large neural networks are trained on massive datasets at scale and often made available to the broader community \cite{varshney2019pretrained,bommasani2021opportunities}. As of early 2026, Hugging Face hosts over 2.5 million pretrained models \cite{huggingface_models}. These models reduce or eliminate the need for training from scratch. By fine-tuning or adapting pretrained weights, users can apply these models to a wide range of downstream tasks across domains such as computer vision, natural language processing, and scientific applications \cite{parisi2022unsurprising,hu2022lora}.

Despite their success, pretrained models introduce significant practical challenges due to their scale \cite{tu2024overview,banyongrakkul2025release}. They often contain hundreds of millions to hundreds of billions of parameters, resulting in substantial storage and memory requirements. This creates a significant bottleneck in resource-constrained environments, including mobile devices, embedded systems, and edge platforms, where computational resources are constrained \cite{lin2022device,liu2024lightweight,emami2025llm}. Consequently, developing efficient compression techniques is essential for enabling the widespread deployment of pretrained models \cite{liang2026comprehensive}.

Existing approaches to model compression can be broadly categorized into four classes: quantization, pruning, knowledge distillation, and low-rank decomposition \cite{gupta2022compression,zhu2024survey,saha2024compressing,liu2025survey}. Quantization methods reduce the memory footprint of a model by representing weights and activations with reduced bit-width. Pruning techniques, on the other hand, aim to remove less important or low-magnitude components of a network, such as individual weights, neurons, or structured blocks, thereby producing sparser, more efficient networks \cite{cheng2024survey}.

Knowledge distillation takes a different approach, where a smaller student model is trained to mimic the behavior of a larger teacher model \cite{cho2019efficacy,wang2026end}. This is  achieved by defining an auxiliary loss function that encourages the student to match the softened output distributions (e.g., softmax probabilities) of the teacher. Knowledge can also be transferred through intermediate feature representations, enabling the student to capture richer structural information from the teacher.

Rounding out this taxonomy, low-rank decomposition methods exploit the observation that many weight matrices in deep neural networks exhibit approximate low-rank structure \cite{idelbayev2020low,feng2022rank,thamm2022random}. A key advantage of this approach is that low-rank matrix factorization, particularly via the singular value decomposition (SVD), is well understood both theoretically and in practice \cite{udell2019big}. 

To illustrate this, consider a weight matrix in the classifier head of a pretrained neural network for image classification tasks, denoted by $W \in \mathbb{R}^{C \times D}$, where $C$ is the number of classes and $D$ is the dimensionality of the input samples after the feature extraction modules. Storing this matrix requires $\mathcal{O}(CD)$ parameters, whereas a rank-$k$ approximation uses only $\mathcal{O}((C + D)k)$ parameters, which can be substantially smaller when $k \ll \min(C, D)$.

However, computing the exact SVD of large weight matrices is computationally expensive and memory-intensive, particularly for modern pretrained models with high-dimensional layers \cite{pourkamali2018randomized}.
In particular, when 
$D>C$, the computational complexity of the SVD scales as $\mathcal{O}(DC^2)$.
This challenge is further amplified when compression must be applied across multiple layers. 

To address this limitation, randomized methods for low-rank approximation have emerged as scalable alternatives to classical deterministic algorithms \cite{martinsson2020randomized,derezinski2024recent}. These methods leverage random projections to capture the dominant subspace of a matrix, enabling the computation of low-rank approximations at significantly reduced computational cost compared to exact SVD.

A prominent example of this class is the randomized SVD (RSVD) \cite{halko2011finding}, which constructs low-dimensional subspace approximations via multiplication of the weight matrix $W$ by a small random projection matrix. A subsequent multiplication by $W^T$ refines this approximation and enables recovery of a low-rank approximation of $W$. RSVD has been widely adopted, including recent work applying it to extract the leading singular vectors of pretrained weight matrices \cite{meng2024pissa}.

While RSVD has been shown to provide accurate low-rank approximations in many applications, its performance can degrade significantly in regimes where the spectral decay is slow, a phenomenon commonly observed in pretrained models. In such cases, the singular value spectrum does not exhibit a sharp drop, making it more challenging for standard RSVD to accurately capture the dominant subspace.

This limitation is illustrated in Figure \ref{fig:demo}(a), where we examine a representative hidden layer from the VGG19 architecture \cite{simonyan2014very} and compute its exact SVD. Although the spectrum initially decays rapidly, it transitions to a much slower decay regime. This behavior is further reflected in Figure \ref{fig:demo}(b), which reports the normalized spectral error, defined as the spectral norm of the approximation error at rank $k$, normalized by the 
$(k+1)$-th singular value. The results demonstrate a noticeable degradation in approximation quality.

\begin{figure}[ht!]
    \centering
\includegraphics[width=0.4\linewidth]{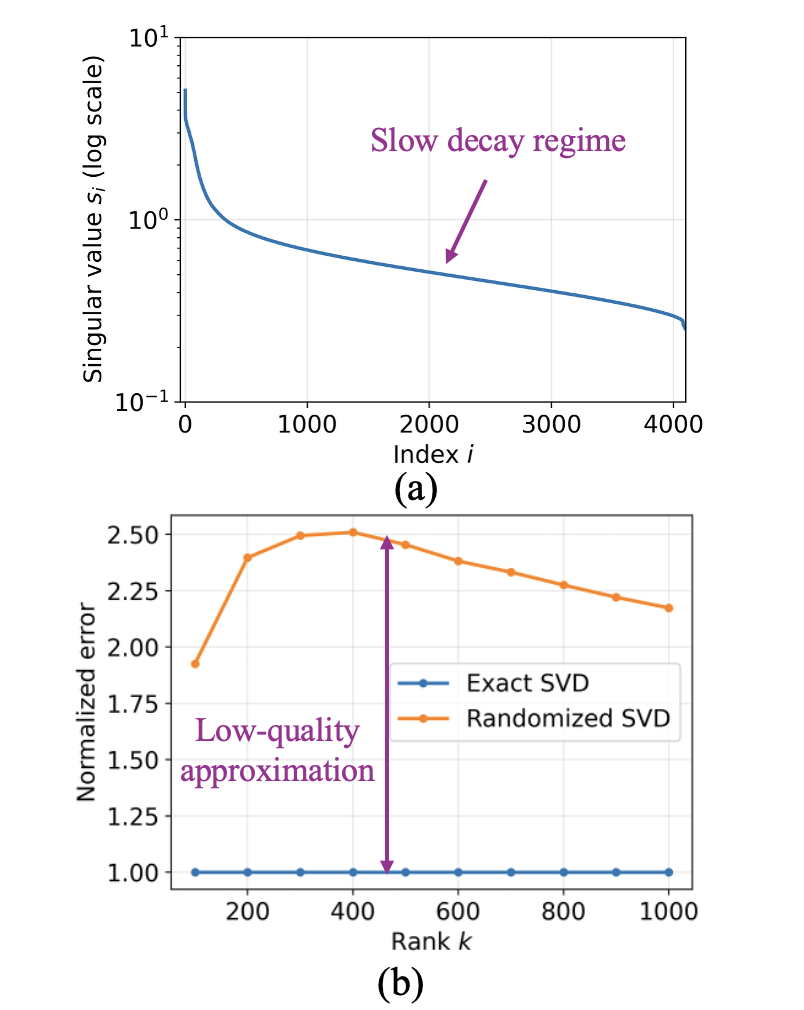}
\vspace{-5mm}
    \caption{Singular value spectrum and normalized spectral error for a layer of VGG with size 4096 $\times$ 25088.}
    \label{fig:demo}
\end{figure}

Motivated by the dual goals of understanding how low-rank approximation quality impacts predictive performance and developing more accurate compression methods, we make two primary contributions in this paper. First, we develop a theoretical analysis of softmax perturbation under low-rank approximation, showing that deviations in predicted class probabilities are bounded by the spectral approximation error of the underlying weight matrices. This result establishes a direct link between approximation quality and classification reliability, offering a principled foundation for analyzing the reliability of compressed models.

Second, we propose randomized subspace iteration (RSI) \cite{tropp2023randomized} as a more accurate and robust alternative to RSVD for compressing pretrained models. By incorporating power iterations, RSI amplifies spectral separation, enabling significantly improved approximation quality in regimes with slow singular value decay, which is common in large pretrained models.



\section{Problem Formulation and Preliminaries.}
Let $W$ be a $C \times D$ matrix, and $D > C$. This assumption is motivated by our focus on compressing the final classification layer, where the feature dimension $D$ is typically larger than the number of classes $C$. However, the discussion and results presented in this paper are general and apply to any matrix-valued ``linear'' layer, and do not depend on the specific dimensions of $W$. 

The matrix $W$ can be written using the singular value decomposition (SVD) as
\begin{equation}
W = U S V^T = \sum_{i=1}^C s_i\, u_i v_i^T.
\end{equation}
Here, $U=[u_1,\ldots,u_C] \in \mathbb{R}^{C \times C}$ and $V=[v_1,\ldots,v_C] \in \mathbb{R}^{D \times C}$ contain the left and right singular vectors, respectively, and $S \in \mathbb{R}^{C \times C}$ is a diagonal matrix with entries $s_1 \geq s_2 \geq \cdots \geq s_C \geq 0$. Each singular value $s_i$ indicates the importance of its corresponding direction. 

The spectral norm of $W$, denoted by $\|W\|_2$, is defined as the largest singular value of $W$, i.e., $\|W\|_2 = s_1$. This norm measures the maximum scaling effect of the linear transformation induced by $W$.

A key property of the SVD is that it provides the optimal low-rank approximation of a matrix. In particular, for any $k < \mathrm{rank}(W)$, the best rank-$k$ approximation of $W$ is obtained by truncating its SVD
\begin{equation}
W_k = \sum_{i=1}^k s_i\, u_i v_i^T.
\end{equation}
This approximation is optimal in the sense that it minimizes the approximation error among all rank-$k$ matrices, i.e.,
\begin{equation}
W_k = \argmin_{\mathrm{rank}(X) \leq k} \|W - X\|_2.
\end{equation}
We are often interested in how much information is lost when using the low-rank approximation
\begin{equation}
W - W_k = \sum_{i>k} s_i u_i v_i^T.
\end{equation}
Hence, for the spectral norm, we have $\|W - W_k\|_2 = s_{k+1}$. This explains why in Figure \ref{fig:demo}, the normalized error for the exact SVD, defined as $\|W - W_k\|_2 / s_{k+1}$, is equal to 1 for all values of $k$.

However, computing the exact SVD is computationally expensive. The cost scales as $\mathcal{O}(DC^2 )$, making it inefficient for large matrices or settings where decompositions must be computed repeatedly (e.g., across layers or multiple models). 
 This motivates the use of randomized methods that approximate the leading singular components of $W$ at a lower computational cost.

 One of the most widely used methods in this class is randomized SVD (RSVD) \cite{halko2011finding}. The main idea is to identify the dominant directions in the range of $W$ by applying the matrix to a set of random vectors. Specifically, we draw a matrix $\Omega \in \mathbb{R}^{D \times k}$, whose entries are sampled from a standard normal distribution; see \cite{saibaba2025randomized} for alternative choices. We then compute
\begin{equation}
X = W \Omega \in \mathbb{R}^{C \times k},
\end{equation}
which captures the action of $W$ on a low-dimensional random subspace and is strongly aligned with its leading left singular vectors. Next, we orthonormalize the columns of $X$ (e.g., via qr decomposition \cite{\cite{pourkamali2019improved}}), and compute
$Y = W^T X \in \mathbb{R}^{D \times k}$.
We then perform the exact SVD of the smaller matrix $Y^T \in \mathbb{R}^{k \times D}$
\begin{equation}
[\hat{U}, \widetilde{S}, \widetilde{V}] = \mathrm{svd}(Y^T),
\end{equation}
and finally set $\widetilde{U} = X \hat{U}$. The resulting factors approximate the leading singular vectors and singular values of the input matrix $W$.  This yields a rank-$k$ approximation of $W$ of the form
$\widetilde{W}_k = \widetilde{U}\,\widetilde{S}\,\widetilde{V}^T$. 

An important observation is that the dominant computational cost of RSVD scales as $\mathcal{O}(C D k)$, which is linear in both dimensions $C$ and $D$ for a fixed target rank. As a result, RSVD is well aligned with modern computing architectures, such as GPUs, which are highly optimized for matrix-matrix operations. This is one of the main reasons behind the popularity of RSVD, which has been widely used across various domains.

However, as illustrated in Figure \ref{fig:demo}, the singular value spectrum of layers in pretrained models often exhibits slow decay. In such regimes, the standard RSVD struggles to produce accurate low-rank approximations, since its performance depends on the separation between leading and trailing singular values. When the spectral decay is slow, the tail singular values remain significant, which leads to larger approximation errors unless the target rank $k$ is chosen sufficiently large \cite{tropp2023randomized}.

Therefore, in the next section, we introduce a more flexible variant of RSVD that enhances low-rank approximation quality by amplifying spectral separation.

\section{Proposed Approach.}
In this work, we consider compressing a pretrained linear layer by replacing its weight matrix $W \in \mathbb{R}^{C \times D}$ with a low-rank approximation. Specifically, we approximate $W$ by a rank-$k$ factorization of the form $AB$, where $A \in \mathbb{R}^{C \times k}$ and $B \in \mathbb{R}^{k \times D}$. This low-rank factorization can be obtained from the singular value decomposition (SVD) $W = U S V^T$ by setting $A = U S^{1/2}$ and $B = S^{1/2} V^T$. When using randomized methods, these factors are approximated accordingly. This formulation enables replacing a single linear layer with two smaller linear layers, thereby reducing the parameter count.

While the truncated SVD provides the optimal rank-$k$ approximation, computing it exactly is computationally expensive for large models. Moreover, as discussed earlier, the singular value spectra of pretrained models often exhibit slow decay. In such cases, RSVD leads to significant loss in approximation quality.

To address this challenge, we propose randomized subspace iteration (RSI), a principled approach that improves low-rank approximation quality by amplifying dominant spectral components while attenuating less significant ones. In addition, we develop a theoretical analysis of softmax perturbations under low-rank approximation, providing further justification for the effectiveness of RSI in compressing pretrained models.

\subsection{Randomized Subspace Iteration}
As discussed earlier, the performance of randomized low-rank approximation methods depends on the decay of the singular values of \( W \). When the spectrum exhibits slow decay, standard RSVD fails to accurately capture the dominant singular subspace.

To address this limitation, we employ randomized subspace iteration (RSI) \cite{gu2015subspace,saibaba2019randomized,tropp2023randomized}, which improves approximation quality by amplifying the contribution of dominant singular values. Recall that 
$W = \sum_{i=1}^{C} s_i u_i v_i^T$ denote the SVD of \( W \). For a random vector \( \omega \sim \mathcal{N}(0, I_D) \) (e.g., a column of \( \Omega \)), we have
\begin{equation}
W \omega = \sum_{i=1}^{C} Z_i s_i u_i,
\end{equation}
where we can show that \( Z_i \sim \mathcal{N}(0,1) \). Applying additional power iterations yields
\begin{equation}
(WW^T)^{q-1} W \omega = \sum_{i=1}^{C} Z_i s_i^{2q-1} u_i.
\end{equation}
This expression highlights that increasing the iteration count \( q \) amplifies the contribution of larger singular values while suppressing smaller ones. 
Notably, when \( q = 1 \), this procedure reduces to standard RSVD. For \( q > 1 \), RSI introduces a controllable mechanism to enhance spectral separation, which is particularly beneficial in regimes with slowly decaying singular values. 

We summarize the RSI procedure in Algorithm~\ref{alg:rsi}. While increasing \( q \) results in additional matrix multiplications involving \( W \) and \( W^T \), this cost is well aligned with modern computing architectures, which are highly optimized for such operations. More importantly, as we demonstrate in the next section, increasing \( q \) can significantly enhance the quality of the low-rank approximation. We further validate this effect through experimental results that highlight the critical role of \( q \).
\begin{algorithm}[h]
\caption{Randomized Subspace Iteration (RSI)} \label{alg:rsi}
\begin{algorithmic}[1]
\REQUIRE Weight matrix $W \in \mathbb{R}^{C \times D}$, target rank $k$, iteration count $q\geq 1$
\STATE Draw a random matrix $\Omega \in \mathbb{R}^{D \times k}$ and set $Y=\Omega$
\FOR{$t = 1, \dots, q$}
    \STATE $X = W Y$
    \STATE $[X, \textunderscore] = \mathrm{qr}(X)$
    \STATE $Y = W^T X$
\ENDFOR
\STATE $[\hat{U}, \widetilde{S}, \widetilde{V}] = \mathrm{svd}(Y^T)$
\STATE $\widetilde{U} = X \hat{U}$
\RETURN $\widetilde{U}, \widetilde{S}, \widetilde{V}$
\end{algorithmic}
\end{algorithm}

\subsection{Theoretical Guarantees.}
We now analyze the impact of low-rank approximation on the output probabilities of a neural network classifier. In particular, we quantify how perturbations in the weight matrix $W$ affect the softmax outputs.

\begin{lemma}[Jacobian of the softmax map]
Let $\sigma:\mathbb{R}^C \to \mathbb{R}^C$ be the softmax function defined by
\begin{equation}
\sigma_i(u)=\frac{e^{u_i}}{\sum_{j=1}^C e^{u_j}},
\qquad i=1,\dots,C.
\end{equation}
Then the Jacobian of $\sigma$ at $u\in\mathbb{R}^C$ is given by
\begin{equation}
J_\sigma(u)=\operatorname{diag}(\sigma(u))-\sigma(u)\sigma(u)^T.\label{eq:jacob}
\end{equation}
\end{lemma}

\begin{proof}
Let
\[
s(u)=\sum_{\ell=1}^C e^{u_\ell},
\qquad
\sigma_i(u)=\frac{e^{u_i}}{s(u)}.
\]
Fix $i,j\in\{1,\dots,C\}$. By the quotient rule,
\[
\frac{\partial \sigma_i(u)}{\partial u_j}
=
\frac{\delta_{ij} e^{u_i}s(u)-e^{u_i}e^{u_j}}{s(u)^2}.
\]
where $\delta_{ij}$ is the Kronecker delta defined by
\[
\delta_{ij} =
\begin{cases}
1, & \text{if } i=j,\\
0, & \text{if } i\neq j.
\end{cases}
\]
Using
\[
\frac{e^{u_i}}{s(u)}=\sigma_i(u),
\qquad
\frac{e^{u_j}}{s(u)}=\sigma_j(u),
\]
we obtain
\[
\frac{\partial \sigma_i(u)}{\partial u_j}
=
\sigma_i(u)\bigl(\delta_{ij}-\sigma_j(u)\bigr).
\]
Stacking these entries yields
\[
J_\sigma(u)=\operatorname{diag}(\sigma(u))-\sigma(u)\sigma(u)^T.
\]
\end{proof}

\begin{theorem}[Softmax perturbation under low-rank approximation]
Let $h:\mathcal{X} \to \mathbb{R}^D$ be a fixed feature extractor and consider
\begin{equation}
z(x) = W h(x) + b, 
\qquad 
\widetilde{z}(x) = \widetilde{W} h(x) + b,
\end{equation}
where $W \in \mathbb{R}^{C \times D}$ is the pretrained weight matrix and $\widetilde{W}$ is a low-rank approximation of $W$. Assume that
\begin{equation}
\|h(x)\|_2 \le R \qquad \text{for all } x \in \mathcal{X}.
\end{equation}
Define the following probability vectors in $\mathbb{R}^C$
\begin{equation}
p(x) = \mathrm{softmax}(z(x)), 
\qquad 
\widetilde{p}(x) = \mathrm{softmax}(\widetilde{z}(x)).
\end{equation}
Then, for every $x \in \mathcal{X}$,
\begin{equation}
\|\widetilde{p}(x) - p(x)\|_\infty
\le
\frac{1}{2} \, R \, \|W - \widetilde{W}\|_2.
\end{equation}
In particular, for every class $c\in\{1,\ldots,C\}$,
\begin{equation}
|\widetilde{p}_c(x) - p_c(x)|
\le
\frac{1}{2} \, R \, \| W- \widetilde{W}\|_2.
\end{equation}
\end{theorem}

\begin{proof}
Let $\Delta W := \widetilde{W} - W$. Since the bias term is the same in both logits, we have
\[
\widetilde{z}(x) - z(x) = \Delta W \, h(x).
\]
Therefore,
\[
\|\widetilde{z}(x) - z(x)\|_\infty
\le
\|\widetilde{z}(x) - z(x)\|_2
=
\|\Delta W \, h(x)\|_2.
\]
Using the properties of the spectral norm,
\[
\|\Delta W \, h(x)\|_2
\le
\|\Delta W\|_2 \, \|h(x)\|_2
\le
R \, \|\Delta W\|_2.
\]
Hence,
\begin{equation}
\|\widetilde{z}(x) - z(x)\|_\infty
\le
R \, \|\widetilde{W} - W\|_2.\label{eq:bound1}
\end{equation}
Let $\sigma:\mathbb{R}^C \to \mathbb{R}^C$ denote the softmax map. Its Jacobian at $u \in \mathbb{R}^C$ is given by \eqref{eq:jacob}
\[
J_\sigma(u) = \operatorname{diag}(\sigma(u)) - \sigma(u)\sigma(u)^T.
\]
For any row $i$, we have
\[
(J_\sigma(u))_{ii} = \sigma_i(u)\bigl(1-\sigma_i(u)\bigr),
\]
and 
\[
(J_\sigma(u))_{ij} = -\sigma_i(u)\sigma_j(u)
\quad (j\neq i).
\]
Thus the absolute row sum is
\[
\sum_{j=1}^C |(J_\sigma(u))_{ij}|
=
\sigma_i(u)\bigl(1-\sigma_i(u)\bigr)
+
\sum_{j\neq i}\sigma_i(u)\sigma_j(u).
\]
Since $\sum_{j\neq i}\sigma_j(u)=1-\sigma_i(u)$, this becomes
\[
2\sigma_i(u)\bigl(1-\sigma_i(u)\bigr).
\]
Now $\sigma_i(u)\in[0,1]$, and
\[
\max_{t\in[0,1]} 2t(1-t)=\frac{1}{2}.
\]
Therefore, for all  $u\in\mathbb{R}^C$  and all  $i$,
\begin{equation}
\sum_{j=1}^C |(J_\sigma(u))_{ij}|
=
2\sigma_i(u)\bigl(1-\sigma_i(u)\bigr)
\le \frac{1}{2}. \label{eq:bound2}
\end{equation}
We now bound each coordinate of $\sigma$ separately. Fix any index $i\in\{1,\ldots,C\}$.
The mean value theorem \cite{rudin1976principles} applied to $\sigma_i$ along the segment from $z(x)$ to $\widetilde{z}(x)$ gives
\[
|\sigma_i(\widetilde{z}(x)) - \sigma_i(z(x))|
=
\bigl|\nabla_u \sigma_i(\xi_i)^\top (\widetilde{z}(x) - z(x))\bigr|,
\]
for some $\xi_i$ on the segment between $z(x)$ and $\widetilde{z}(x)$. Applying the dual
pairing $|\langle g,v\rangle| \le \|g\|_1\|v\|_\infty$,
\[
|\sigma_i(\widetilde{z}(x)) - \sigma_i(z(x))|
\le
\|\nabla_u \sigma_i(\xi_i)\|_1 \, \|\widetilde{z}(x) - z(x)\|_\infty.
\]
Since $\|\nabla_u \sigma_i(u)\|_1$ is the $i$-th absolute row sum of $J_\sigma(u)$, the bound above and the one in \eqref{eq:bound2} lead to
\[
|\sigma_i(\widetilde{z}(x)) - \sigma_i(z(x))|
\le
\frac{1}{2}\|\widetilde{z}(x) - z(x)\|_\infty.
\]
Taking the maximum over all $i\in\{1,\ldots,C\}$,
\begin{equation}
\|\sigma(\widetilde{z}(x)) - \sigma(z(x))\|_\infty
\le
\frac{1}{2}\|\widetilde{z}(x) - z(x)\|_\infty. \label{eq:bound3}
\end{equation}
Hence, using \eqref{eq:bound1} and \eqref{eq:bound3}, we get
\[
\|\widetilde{p}(x)-p(x)\|_\infty
\le
\frac{1}{2} \,\|\widetilde{z}(x)-z(x)\|_\infty
\le
\frac{1}{2} \, R \,\|\widetilde{W}-W\|_2.
\]
This proves the theorem.
\end{proof}

\begin{remark}[Interpretation of the spectral error term]
The bound above shows that the perturbation in the predicted class probabilities is controlled directly by the spectral approximation error $\|W-\widetilde W\|_2$.

If $\widetilde W$ is the best rank-$k$ approximation to $W$ obtained from the truncated singular value decomposition,
\[
\|W-\widetilde W_k\|_2 = s_{k+1},
\]
where $s_{k+1}$ denotes the $(k+1)$-st singular value of $W$. Therefore,
\begin{equation}
\|\widetilde p(x)-p(x)\|_\infty
\le
\frac{1}{2} R\,s_{k+1}.
\end{equation}

When $\widetilde W$ is produced by RSI, Theorem 9.1 of \cite{tropp2023randomized} provides a bound on the expected approximation error. In particular, the logarithm of the normalized expected error decays at rate $O(1/m)$, where $m$ is the number of matrix multiplications with $W$ and $W^T$. This implies that increasing the number of iterations $q$ improves the quality of the recovered low-rank approximation. More precisely, RSI yields a bound of the form
\begin{equation}
\log\!\left(\frac{\mathbb{E}\|W-\widetilde W\|_2^2}{s_{k+1}^2}\right)
\le
\frac{1}{m-1}\log H,
\end{equation}
where $H > 1$ mainly depends on the spectrum of $W$, and the expectation $\mathbb{E}$ is taken with respect to the randomness in the algorithm (i.e., drawing the random matrix $\Omega$). Exponentiating both sides gives
\[
\mathbb{E}\|W-\widetilde W\|_2^2
\le
s_{k+1}^2 \, H^{\frac{1}{m-1}}.
\]
Thus, as $m$ increases, the factor $H^{1/(m-1)}$ approaches $1$, implying that the randomized approximation error converges to the optimal error. 
\end{remark}

\section{Experiments.} In this section, we evaluate the effectiveness of randomized low-rank approximation methods for compressing pretrained vision models. We consider two representative architectures: VGG19 and ViT-B/32 \cite{dosovitskiy2021an} . These pretrained models are obtained from the PyTorch Hub \cite{pytorch_hub}. This selection allows us to study both a traditional convolutional neural network and a modern transformer-based architecture.

The VGG19 model consists of convolutional feature extraction modules followed by a classifier composed of 3 fully connected (linear) layers. In contrast, ViT-B/32 (Vision Transformer Base with 32$\times$32 patch size) is a transformer-based architecture that processes images as sequences of patches and relies heavily on linear layers within its encoder blocks, including projection and feedforward components. As such, it provides a  more modern setting for evaluating randomized methods.

We begin our analysis by focusing on a single representative layer from each model. Specifically, we examine the approximation quality of randomized SVD (RSVD), which corresponds to randomized subspace iteration (RSI) with iteration count $q=1$, and compare it with RSI for higher iteration counts $q\in\{2,3,4\}$ (see Algorithm \ref{alg:rsi}). Hence, we systematically investigate the trade-off between approximation quality and computational cost as controlled by the parameter 
$q$. All experiments are conducted on a single A100 GPU.

In the second phase of our experiments, we extend this analysis to an end-to-end compression setting by applying low-rank approximation to all linear layers within each model. We then evaluate the resulting trade-offs between predictive performance and model compression, where compression is quantified by the reduction in the number of parameters. This provides a comprehensive assessment of the practical impact of randomized low-rank methods, demonstrating their applicability beyond individual layers to full model compression. For evaluation, we use the Imagenette dataset, which is a curated subset of ImageNet containing 10 classes. 

Importantly, our experiments do not involve any retraining or fine-tuning. The goal is to evaluate whether a compressed version of a pretrained model can be used directly as an off-the-shelf classifier on similar test data, thereby isolating the effect of low-rank approximation on predictive performance.

\subsection{Single-Layer Analysis.}
We begin by analyzing the largest linear layer in the classifier module of \textsc{VGG19}, which has dimensions \(4{,}096 \times 25{,}088\) and approximately 102.76M parameters. This layer provides a challenging and representative setting for evaluating randomized low-rank approximation methods.

Since randomized methods involve sampling random matrices, we repeat each experiment 20 times and report the mean normalized error and runtime for each rank \(k\). The normalized error is defined as the spectral norm of the approximation error at rank \(k\), normalized by the \((k+1)\)-th singular value. By optimality of the SVD, this quantity is equal to 1 for the exact SVD. Thus, our goal is to assess whether increasing the iteration parameter \(q\) yields errors closer to this value.

For runtime, we compute the exact SVD once, as the full decomposition enables efficient construction of any rank-\(k\) approximation via matrix multiplications. In contrast, randomized methods construct the approximation directly for each \(k\), resulting in rank-dependent computational cost.

As shown in Figure~\ref{fig:layer_vgg}(a), increasing the iteration count to \(q=2\), which introduces an additional step of attenuating smaller singular values, leads to a significant improvement in approximation quality. While the normalized error is approximately 2 for RSVD (i.e., RSI with \(q=1\)), it drops to around 1.3 when \(q=2\). Moreover, the error continues to decrease as \(q\) increases, although the marginal gains become smaller. Notably, for \(q=4\), the normalized error is approximately 1.1, indicating that the approximation is very close to the optimal value of 1.

\begin{figure}[htbp!]
    \centering    \includegraphics[width=0.4\linewidth]{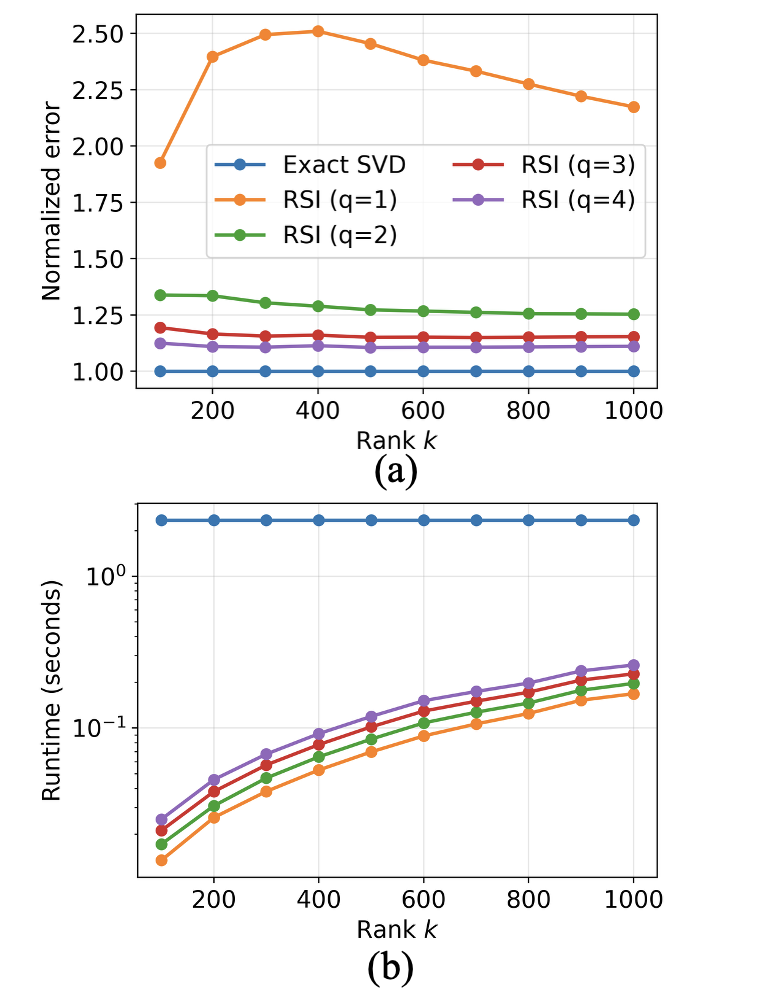}
    \vspace{-5mm}
    \caption{Normalized error and runtime for low-rank approximation of a single layer in VGG19, comparing SVD and RSI across ranks $k$ and iteration counts $q$.}
    \label{fig:layer_vgg}
\end{figure}

Additionally, Figure~\ref{fig:layer_vgg}(b) reports the runtime comparison. Computing the exact SVD on the GPU requires approximately 2.33 seconds. In contrast, RSI achieves substantial computational savings across all ranks. At \(k=200\), RSI with \(q=2\) and \(q=4\) are approximately \(76\times\) and \(51\times\) faster, respectively. Notably, even at the largest rank considered (\(k=1000\)), RSI maintains a speedup of approximately one order of magnitude. These results highlight that RSI provides significant efficiency gains while delivering near-optimal approximation quality.

Next, we consider a linear layer from the ViT architecture, taken from one of its encoder blocks. In transformer models, such layers are part of the feedforward module, where the input features are first expanded to a higher-dimensional space and then projected back, enabling richer feature representations. The corresponding weight matrix has dimensions \(768 \times 3{,}072\), resulting in approximately 2.3M parameters, which is smaller than the layer considered in VGG19. However, transformer-based architectures contain many such linear layers; in this model, there are a total of 37 linear layers. Consequently, compressing these layers can lead to a substantial reduction in overall model size.

\begin{figure}[htbp!]
    \centering
\includegraphics[width=0.44\linewidth]{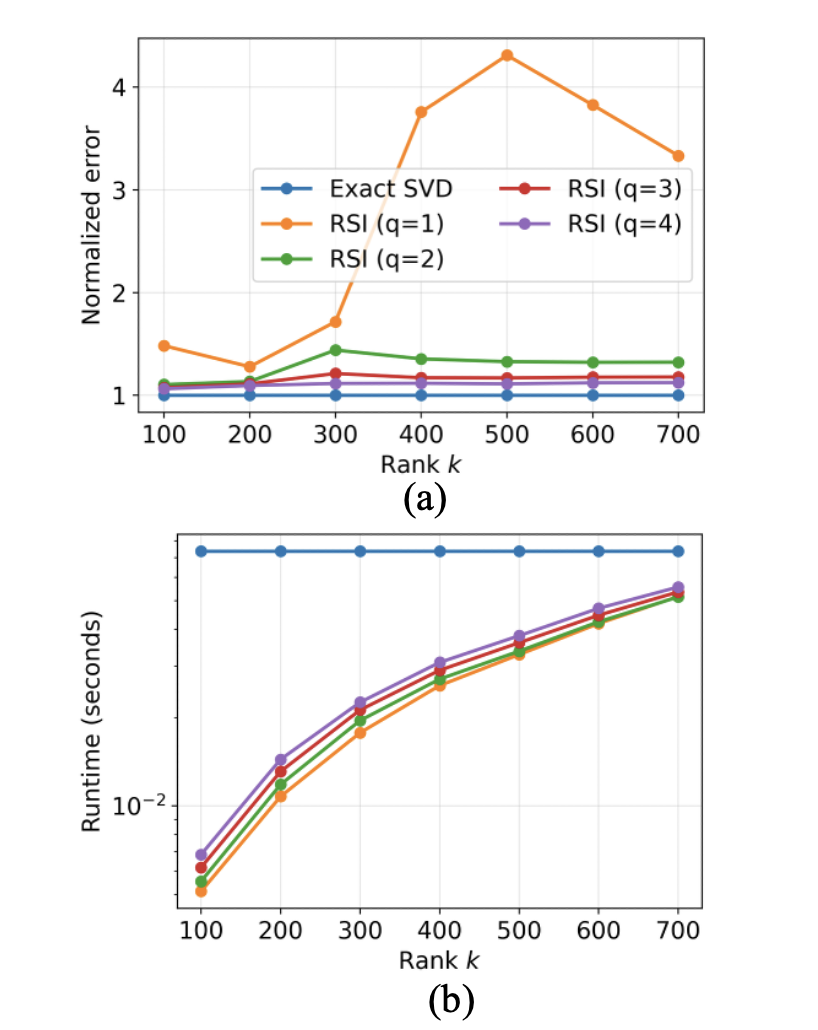}\vspace{-5mm}\caption{Normalized error and runtime for low-rank approximation of a single layer in ViT.}
\label{fig:layer_vit}
\end{figure}

Based on Figure~\ref{fig:layer_vit}(a), we observe that RSVD (equivalently, RSI with \(q=1\)) fails to provide accurate approximations across ranks. In particular, the normalized error exceeds 4 for \(k=500\), highlighting the need for increasing the iteration parameter \(q\). In contrast, for \(q=3\) and \(q=4\), the normalized error consistently remains below 1.2, demonstrating that RSI can achieve accurate low-rank approximations.

\begin{table*}[t]
\centering
\caption{End-to-end compression results for VGG19 and ViT. We report compression time (seconds), compression ratio, and Top-1/Top-5 accuracy for different compression factors \(\alpha\) and RSI iteration counts \(q\).}
\label{tab:compression_results}
\begin{tabular}{cccccc|cccccc}
\toprule
\multicolumn{6}{c|}{\textbf{VGG19}} & \multicolumn{6}{c}{\textbf{ViT-B/32}} \\
\midrule
\(\alpha\) & \(q\) & Time & Ratio & Top-1 & Top-5 &
\(\alpha\) & \(q\) & Time & Ratio & Top-1 & Top-5 \\
\midrule

0.8 & 1 & 3.48  & 1.02  & 82.40\% & 96.33\% & 0.8 & 1 & 1.63 & 1.01 & 86.98\% & 98.15\% \\
0.8 & 2 & 3.53 & 1.02 & 82.60\% & 96.47\% & 0.8 & 2 & 1.65 & 1.01 & 89.19\% & 98.48\% \\
0.8 & 3 & 3.60 & 1.02 & 82.65\% & 96.49\% & 0.8 & 3 & 1.73 & 1.01 & 89.59\% & 98.52\% \\
0.8 & 4 & 3.67 & 1.02 & 82.67\% & 96.50\% & 0.8 & 4 & 1.81 & 1.01 & 89.73\% & 98.55\% \\
\cmidrule(lr){1-6} \cmidrule(lr){7-12}

0.6 & 1 & 2.17 & 0.80 & 80.50\% & 95.75\% & 0.6 & 1 & 1.13 & 0.84 & 76.17\% & 93.97\% \\
0.6 & 2 & 2.26  & 0.80 & 82.13\% &  96.27\% & 0.6 & 2 & 1.16 & 0.84 & 85.65\% & 97.81\% \\
0.6 & 3 & 2.33 & 0.80 & 82.41\% & 96.29\% & 0.6 & 3 & 1.23 & 0.84 & 86.21\% & 97.96\% \\
0.6 & 4 & 2.41 & 0.80 & 82.48\% & 96.30\% & 0.6 & 4 & 1.29 & 0.84 & 86.37\% & 98.01\% \\
\cmidrule(lr){1-6} \cmidrule(lr){7-12}

0.4 & 1 & 1.18 & 0.58 & 75.02\% & 94.19\% & 0.4 & 1 & 0.71 & 0.68 & 16.82\% &  35.43\%\\
0.4 & 2 & 1.23 & 0.58 & 81.17\% & 95.95\% & 0.4 & 2 & 0.75 & 0.68 & 74.05\% & 93.76\% \\
0.4 & 3 & 1.30  & 0.58 & 81.47\% & 96.07\% & 0.4 & 3 & 0.79 & 0.68 & 76.56\% & 95.18\% \\
0.4 & 4 & 1.37 & 0.58 & 81.53\% & 96.07\% & \cellcolor{gray!15} 0.4 & \cellcolor{gray!15} 4 & \cellcolor{gray!15} 0.86 & \cellcolor{gray!15} 0.68 & \cellcolor{gray!15} 77.15\% & \cellcolor{gray!15} 95.54\% \\
\cmidrule(lr){1-6} \cmidrule(lr){7-12}

 0.2 &  1 &  0.49 &  0.36 &  59.27\%  &  87.36\% & 0.2 & 1 &  0.34& 0.51 & 0.07\% & 0.62\% \\
0.2 & 2 & 0.53 & 0.36 & 77.50\% & 94.86\% & 0.2 & 2 & 0.37 & 0.51 &1.96\%  & 7.52\% \\
0.2 & 3 & 0.57 & 0.36 & 78.39\% & 95.08\% & 0.2 & 3 & 0.39 & 0.51 & 4.87\% & 14.90\% \\
\cellcolor{gray!15} 0.2 & \cellcolor{gray!15} 4 & \cellcolor{gray!15} 0.61 & \cellcolor{gray!15} 0.36 & \cellcolor{gray!15} 78.63\% & \cellcolor{gray!15} 95.18\% & 0.2 & 4 & 0.43 & 0.51 & 6.43\% & 17.92\% \\

\bottomrule
\end{tabular}
\end{table*}

Figure~\ref{fig:layer_vit}(b) reports the runtime comparison. Despite the relatively small matrix size, RSI remains significantly faster than the exact SVD. The exact SVD requires approximately 0.07 seconds, whereas RSI with \(q=4\) and \(k=100\) takes about 0.007 seconds, corresponding to an order-of-magnitude speedup. These results again indicate that RSI with \(q>1\) achieves near-optimal approximation quality while offering substantial computational savings, particularly when the target rank \(k\) is much smaller than the matrix dimensions.

\subsection{End-to-End Compression of VGG19 and Vision Transformers.}
We present comprehensive results on the end-to-end compression of the two pretrained models. Both models are originally trained on the ImageNet dataset, which consists of \(C = 1{,}000\) classes. Although the evaluation dataset contains only a subset of 10 classes, we retain the original output layer with 1,000 neurons to faithfully assess the predictive performance of the pretrained models after compression. 

Given the large number of classes, it is standard to report Top-1 and Top-5 accuracy (in percentage). Top-1 accuracy measures the proportion of test samples for which the model's most confident prediction matches the true label, while Top-5 accuracy measures the proportion of test samples for which the true label appears among the model's five most confident predictions.

We apply RSI with varying iteration counts \(q\) to all linear layers. Specifically, VGG19 contains 3 such layers, while the ViT architecture includes 37 linear layers. Since these layers have different dimensions, we introduce a compression parameter \(0 < \alpha < 1\) to control the target rank for each weight matrix of size \(C \times D\), defined as \(k = \lceil \alpha\cdot \min(C, D) \rceil\). Smaller values of \(\alpha\) correspond to more aggressive compression and greater parameter reduction. Also, we note that for larger values of \(\alpha\), the number of parameters after low-rank approximation may slightly increase, as the resulting rank \(k\) approaches the original matrix dimensions.

The end-to-end compression results are summarized in Table~\ref{tab:compression_results}, where we report the total compression time across all linear layers, the compression ratio (defined as the fraction of parameters in the compressed model relative to the original model), and the Top-1 and Top-5 accuracy. For reference, the uncompressed models achieve Top-1/Top-5 accuracies of 82.57\%/96.51\% for VGG19 and 90.55\%/98.68\% for ViT.

For \(\alpha = 0.8\), all values of \(q\) yield strong performance, which is expected as this setting results in minimal compression. As \(\alpha\) decreases, the role of the iteration parameter \(q\) becomes more pronounced. In particular, for VGG19 at \(\alpha = 0.2\), RSI with \(q>1\) significantly improves performance, whereas RSVD (i.e., \(q=1\)) suffers from a substantial drop in accuracy. Overall, \(\alpha = 0.2\) and \(q = 4\) provide an effective trade-off for VGG19, achieving notable compression with limited degradation in predictive performance.

For the ViT model, we observe similar trends with respect to \(q\), where increasing the number of iterations consistently improves accuracy. However, $\alpha = 0.2$ leads to a noticeable performance drop. A more favorable trade-off is achieved at \(\alpha = 0.4\), where the reduction in Top-5 accuracy is limited to approximately 3\%.

Across both models, the total time required to compress all linear layers is under one second for the best-performing configurations. This highlights the computational efficiency of randomized methods, demonstrating their ability to deliver substantial compression with minimal overhead.

Finally, we note that our approach is complementary to low-rank adaptation methods such as LoRA \cite{hu2022lora}. While LoRA reduces the number of trainable parameters during fine-tuning by constraining updates to a low-rank subspace, it still operates on the full pretrained weight matrices. RSI can be applied to compress these backbone weights prior to or during adaptation, offering additional efficiency gains in memory and computation. This suggests a natural combination: applying RSI to obtain a compressed backbone, then using LoRA-style adaptation on top, yielding benefits from both directions.

\section{Conclusion.}
In this work, we presented both theoretical and empirical evidence demonstrating the advantages of randomized subspace iteration (RSI) over the widely used randomized singular value decomposition (RSVD). In particular, RSI achieves substantially improved approximation quality, leading to superior end-to-end compression performance for pretrained models under more aggressive compression regimes. Our findings suggest several promising directions for future research. First, developing adaptive strategies for selecting layer-wise ranks is especially important for transformer-based architectures, which contain a large number of linear layers with varying characteristics. Second, randomized low-rank approximations of pretrained weights offer a natural framework for the low-rank adaptation of pretrained models. Third, while our experiments focused on vision models, the proposed approach can be extended to a broader class of architectures, including vision-language and other multimodal models. In such settings, enabling modality-specific compression ratios may provide an effective pathway toward improving computational efficiency.

\bibliographystyle{siamplain}
\bibliography{example_references}
\end{document}